\documentclass[anon,12pt]{amsart}
\title{Linear regression in $p$-adic metric spaces}
\author[G. D. Baker]{Gregory D. Baker}
\address{School of Computing, Australian National University, 108 North Road Acton, ACT 2601 Australia}
\email{greg.baker@anu.edu.au}
\author[S. McCallum]{Scott McCallum}
\address{School of Computing, Macquarie University, Macquarie Park NSW 2109 Australia}
\email{scott.mccallum@mq.edu.au}
\author[D. A. Pattinson]{Dirk Pattinson}
\address{School of Computing, Australian National University, 108 North Road Acton, ACT 2601 Australia}
\email{dirk.pattinson@anu.edu.au}

\usepackage[utf8]{inputenc}
\usepackage{mathtools}
\usepackage{amsmath,amssymb,amsfonts,amsthm}
\theoremstyle{plain}
\newtheorem{thm}{Theorem}

\newtheorem{corollary}[thm]{Corollary}

\newtheorem{remark}[thm]{Remark}

\newtheorem{notn}{Notation}

\newtheorem{problem}{Problem}

\newcommand{\Q}{\mathbb{Q}}

\usepackage{geometry}
\usepackage{framed}
\usepackage{tikz}
\usepackage{todonotes}
\usepackage{booktabs}
\usetikzlibrary{graphs} 
\usepackage[natbib=true,style=numeric-comp]{biblatex}
\newcommand{\mycite}[1]{\cite{#1}}

\usepackage{setspace}
\definecolor{bluegray}{rgb}{0.4, 0.6, 0.8}

\usepackage[toc,page]{appendix}

\begin{document}

\begin{abstract}

Many real-world machine learning problems involve inherently hierarchical data, 
yet traditional approaches rely on Euclidean metrics that fail to capture the 
discrete, branching nature of hierarchical relationships. We present a 
theoretical foundation for machine learning in $p$-adic metric spaces, which 
naturally respect hierarchical structure. Our main result proves that an 
$n$-dimensional plane minimizing the $p$-adic sum of distances to points in a dataset 
must pass through at least $n+1$ of those points --- a striking contrast to Euclidean 
regression that highlights how $p$-adic metrics better align with the discrete 
nature of hierarchical data.

  As
  a corollary, a polynomial of degree $n$ constructed to minimise the
  $p$-adic sum of residuals will pass through at least $n+1$
  points. As a further corollary, a polynomial of degree $n$
  approximating a higher degree polynomial at a finite number of
  points will yield a difference polynomial that has distinct rational
  roots.

We demonstrate the practical significance of this result through two applications
in natural language processing: analyzing hierarchical taxonomies and modeling
grammatical morphology.  These results suggest that $p$-adic
metrics may be fundamental to properly handling hierarchical data structures in 
machine learning. In hierarchical data, interpolation between points often makes less sense
than selecting actual observed points as representatives.

\end{abstract}
\keywords{machine learning, $p$-adic geometry, grammatical morphology}
\subjclass[2010]{11D88,62J99,68T50}
\maketitle

\section{Introduction}

Machine learning has overwhelmingly relied on Euclidean metrics, implicitly 
assuming that data exists in a continuous space where small changes yield 
proportionally small differences. Yet many real-world problems - from biological 
taxonomies to grammatical structures - are inherently hierarchical, where 
similarity is better measured by proximity in a tree rather than distance in 
a continuous space.

This fundamental mismatch between Euclidean metrics and hierarchical data 
has profound implications. When analyzing hierarchical structures, two points 
that appear close in Euclidean space may be very distant in terms of their 
relationship within the hierarchy, and vice versa. Consider biological 
classification: a whale and a fish may appear similar in many measurable 
dimensions, yet are vastly different in their taxonomic relationship.

Euclidean thinking permeates machine learning, though.
One of the more important tasks in the formulation of a machine learning problem is
finding an appropriate loss function to minimise. Typically we do this
by embedding the data into a Euclidean space and using a loss function
that is implicitly Euclidean.

Some examples of explicitly and implicitly Euclidean loss functions:

\begin{description}
    \item[Explicit] The ${L^2}$ norm --- the loss function is a residual sum of 
squares of the differences between a predicted value and a ground truth value. 
\item[Implicit] The loss function is a cross-entropy loss for a prediction.
\end{description}

Many other loss functions --- the ${L^1}$ norm,
the Manhattan distance --- can be approximated with an Euclidean distance.

When would Euclidean distance {\it not} be an appropriate loss function?

\begin{itemize}
    \item When the problem is predicting a position on a
hierarchical tree, the loss function will have to reflect the distance
away from the correct position in the tree. For example, the distance 
between two points could be the depth of their nearest common ancestor. 
\item When the problem is trying to
predict a polynomial, an appropriate loss function may be the degree
of the residual. For example, if the correct answer is $x^2$, then $x^2+1$ is
likely to be a better answer than $x$, even though the former
polynomial has no overlap with the target and the latter polynomial
intersects it.
\end{itemize}

Why could these not be turned into problems with a Euclidean loss?

Asking whether a problem could be represented accurately with a Euclidean loss
function is asking whether an isometry exists between the relevant distance
metric (common ancestor depth, or polynomial degree) and the Euclidean metric.
It is a fundamental result in topology, dating back to Hausdorff's formalisation of topological spaces \mycite{hausdorff1914grundzuge}, that invariants such as connectedness, compactness, and dimension are preserved under homeomorphisms. Since isometries induce homeomorphisms in metric spaces, they necessarily preserve these topological properties.
The two examples given are totally disconnected spaces (as proven in Problem 63 in \mycite{gouvea_padic}, for example), unlike Euclidean spaces where any two points can be connected with a continuous path.

\subsection{Structure of the paper}

We focus on the $p$-adic metric in this paper as it is an example of a
highly non-Euclidean metric, and explore its implications for
machine learning. We show that intuitions from Euclidean
linear regression are unhelpful for $p$-adic linear regression, and
then prove a useful foundational theorem about $p$-adic linear
regression --- which would be false in a Euclidean space --- to illuminate some of the strangeness of $p$-adic machine
learning. 
Having proven the theorem, we use this to create an algorithm for
solving $p$-adic linear regression problems. We make some attempts at
optimisation, but observe that there is scope for further research to improve
its efficiency.
We then 
 explore two case studies (both involving language processing tasks where language
or grammar is modelled hierarchically) to show that $p$-adic
linear regression can be used to solve some problems in unusual
and interesting ways.
We conclude with a section of unsolved and open problems
that arose in the writing of this paper.


This paper was inspired by a question posed  by Igor Shparlinski \mycite{shparlinski2024private}, who asked whether multivariate $p$-adic regression can be
solved similarly to its one-dimensional counterpart \mycite{aaclpadiclinear}. We provide a
positive answer, with rigorous proofs showing that an optimal plane in
$p$-adic space must pass through at least $n+1$ points in a dataset.
A search of the literature turns up
no other related work on $p$-adic linear regression. However, research has been done on other machine learning algorithms where distance
is measured $p$-adically, such as: Murtagh \mycite{murtagh-model}, mainly looking at nearest neighbour methods; and Khrennikov
\mycite{khrennikov-neural} 
on neural networks.



\section{A brief overview of $p$-adic numbers and $p$-adic spaces}

%


Kurt Hensel (in the late 19th century) observed that there is an
unusual family of distance functions that have useful properties.

Given a prime number \( p \) and a non-zero rational number \( x = \frac{a}{b} \) where \( a \) and \( b \) are integers, the \( p \)-{\em adic} {\em valuation} \( v_p(x) \) is defined as the highest power of \( p \) that divides \( a \) minus the highest power that divides \( b \). This can be positive, zero or negative. The \( p \)-{\em adic} {\em absolute value} \( |x|_p \) is then given by:

\[ |x|_p = p^{-v_p(x)} \]

For \( x = 0 \), we define \( |x|_p = 0 \).
For $x$ and $y$ both rational, the $p$-{\em adic} {\em distance} $d(x,y)$ between $x$ and $y$ is then $|x - y|_p$; and the function $d$ of $x$ and $y$ thereby determined is also called the $p$-{\em adic} {\em metric}.

According to the $p$-adic notion of distance,
two rational numbers are close together if
their difference is highly (and positively) divisible by the prime $p$. 3-adically, 1 and 4 are close
together. 1 and 10 are very close together, because their difference can be divided by
3, and then divided by 3 yet again. 1 and 28 are closer still (3-adically). However, $\frac{2}{3^{10}}$ and $\frac{1}{3^{10}}$ are not close, 3-adically.

A little arithmetic and algebra can convince the reader that the following properties of the $p$-adic absolute value hold for all prime numbers $p$:

\begin{description}
\item[Non-negativity] $|x|_p \ge 0$
\item[Positive definiteness] $|x|_p = 0$ if and only if $x = 0$
\item[Multiplicativity] \( |x y|_p = |x|_p |y|_p \)
\item[Triangle inequality] \( |x + y|_p \le |x|_p + |y|_p \)
\end{description}

Analogues of the above hold for the familiar absolute value on the reals ($\mathbb{R}$).

Ostrowski proved \mycite{ostrowski1916} that every non-trivial absolute value over
the rationals --- that is, a function for which the above four properties hold --- is either a positive power of the standard Euclidean absolute value, or a positive power of the $p$-adic absolute value.

It follows from the above that $\mathbb{Q}$, together with the $p$-adic distance function $d$, constitutes a {\em metric space}.
(This formally justifies the terminology ``$p$-adic metric" mentioned above.)

The $p$-adic absolute value (or metric) actually has a slightly stronger version of the triangle inequality (aptly called ``the strong triangle inequality''):
\begin{equation}
 |x + y|_p \le \max(|x|_p ,|y|_p)
  \label{strongtriangle}
\end{equation}

It follows that $(\mathbb{Q}, d)$ moreover constitutes an {\em ultrametric space}, with {\em ultrametric} $d$.

There are some unfamiliar aspects of the $p$-adic absolute value (and metric). 
Famously, every point inside a circle
is a centre of that circle.

Another example:
it is not possible to get from $1$ to
$2$ in small steps. $1+p$ is close to 1, but it is neither closer nor further from $2$ than
$1$ was. $\frac{3}{2}$ is not half-way between them: 2-adically, 
$\frac{3}{2}$ is further
from $1$ and $2$ than they are from each other. 





\section{Multivariate $p$-adic linear regression}\label{least-squares-comparison}





There is a small discrepancy in naming conventions between machine learning
and linear algebra. A linear regression problem in machine learning (and in
statistics) is a search for an {\it affine} function, not a linear function.
We may therefore state our multivariate $p$-adic linear regression
problem as follows:

\begin{problem} \label{prob:lin-reg}
Given $X_1, \ldots, X_k \in \mathbb{Q}^n$ and $y_1, \ldots, y_k \in \mathbb{Q}$, 
find an affine function $F:
\Q^n \to \Q$ that minimises a loss function defined by $\sum_{i=1}^k |F(X_i) -
y_i|_p$.
\end{problem}

Note that we could generalise this problem to cover any regression
problem over any field which has an ultrametric valuation. However,
for concreteness 
we shall refrain from such generalisation and
work with the $p$-adic
numbers. The proof of our main theorem (Theorem \ref{core-theorem} in Section \ref{core-theorem-section})  follows  
{\it mutatis mutandis} 
for other ultrametric
valuations.


  \begin{notn}
  We identify vectors in $\Q^n$ with $1 \times n$ matrices, i.e.
  we  conceive of vectors as row vectors. Given $X = (x_1, \dots,
  x_n)$ and $Y = (y_1, \dots, y_n) \in \Q^n$, $X \cdot Y = \sum_{1
  \leq i\leq n} x_i y_i$ denotes the dot product of $X$ and $Y$. 
  We use standard notation
  for matrices, and write $(\cdot)^T$ for the matrix transpose.
  If $X = (x_1, \dots, x_{n}) \in \Q^n$, we let $(X, 1) =
  (x_1, \dots, x_{n}, 1)$. 
  \end{notn}

  Given that linear functions can be represented by matrix operations, our
  problem can now be reformulated as follows: 

  \begin{problem}
  Given $X_1, \dots, X_k \in \Q^n$ and $y_1, \dots, y_k \in \Q$,
find a vector $V \in \Q^{n+1}$ that minimises
$\sum_{i=1}^k |V \cdot (X_i, 1) - y_i |_p$.
\label{formalnotation}
  \end{problem}




We note that $p$-adic regression shares a formulation that is similar
to ordinary least squares regression. Ordinary least squares can be
solved in closed form analytically by taking the derivative of the
cost function and finding the sole zero of the derivative.  This is
possible because the cost function is convex and has a single global
minimum.

Unfortunately, $p$-adic linear regression is not as simple, as discussed in the next subsections.


\subsection{No guarantee of a global minimum}

 The loss function of a $p$-adic linear regression problem does not always have a single
  global minimum.
There can be multiple global minima even in the lowest-dimensional
dataset with small numbers of points.

Consider the following
dataset 
where there are four equally good lines of best fit
2-adically:
\begin{equation*}\label{padic-four-solution}
  \begin{matrix}
    (0, 0) &  (1, 0) &
    (1, 1) &  (1, 2) &
    (1, 3)
  \end{matrix}
\end{equation*}

The 2-adic sum of distances from those points is $\frac{5}{2}$
for all of the following lines:
$y = 0$,  $y = x_1$,  $y = 2x_1$ and $y = 3x_1$.

By the theorem in \mycite{aaclpadiclinear}, the optimal line must pass through
at least two points in the dataset.
  
  Enumerating all six of the other possible lines that pass through
  two points in the dataset finds no lines with a lower loss than $\frac{5}{2}$.

\subsection{Structure and repetition of good solutions}

Consider
 $X_{i} = y_i = i - 1$ for $i \in \{1, 2 \ldots 5\}$.  Or equivalently,
 the set of $(X_{i}, y_i)$ pairs:
\begin{equation*}\label{trivial}
  \begin{matrix}
    (0, 0) &  (1, 1) &
    (2, 2) &  (3, 3) &
    (4, 4)
  \end{matrix}
\end{equation*}
where the $X$ and $y$ values are identical.
 This obviously
has a line of best fit $y = x$, with residual sum equal to zero.
If we are minimising the 3-adic distance, then 
 $y = x + 1$ has a residual sum of 5 and the lines
 $y = 2x$, $y = 3x$ and $y = 5x$ all have a 
residual sum of $\frac{10}{3}$ --- clearly worse lines than $y = x$.

But note that $y = x + 3$ has a quite small residual sum of $\frac{5}{3}$.
The line $y = 4x$ is quite small too, with a residual sum of $\frac{10}{9}$.
These are obviously quite good, and in a moment we can show that they are local minima.

$y = 10x$ is better still (because $10x = 9x + 1x$ and $9 = 3^2$) with a residual
sum of $\frac{10}{27}$. 

The pattern is that $y = (p^t + 1) x$ is a very good line for all
$t \in \mathbb{Z}^+$. The line $y = (3^{1000000}+1)x$ is very nearly
as good a line of best fit as $y = x$ is.

 Starting with a global minimum,
we can find a local minimum by adding any integer multiple of $p$
to any coefficient in the linear equation. We can find a very good (nearly
globally optimal) local minimum by adding an integer multiple of $p^t$ where
$t$ is very large.

Thus, there are an infinite number of local minima.
The implication is that a random starting point has an absurdly
high probability of landing near a local minimum rather than a global
one.


\subsection{Gradient descent is not viable}

Machine learning algorithms that use $\mathbb{R}$ instead of
$\mathbb{Q}_p$ often use gradient descent to find solutions where the
loss landscape may contain multiple global minima and many local
minima, so it is reasonable to ask if it could be applied to $p$-adic
machine learning. Unfortunately it is not, since a loss function
constructed using a $p$-adic norm is locally constant almost
everywhere.

Using the notation of Problem \ref{formalnotation}, let $V, V' \in \mathbb{Q}^{n+1}$
be ``close'' in the sense that if we define

\begin{equation}
  \epsilon_i = V \cdot (X_i,1) - V' \cdot (X_i,1)
  \label{epsilon-definition}
  \end{equation}

then 

\begin{equation}
  \left| \epsilon_i \right|_p \leq \left| V \cdot (X_i,1) - y_i \right|_p
  \label{epsilon-setup}
\end{equation}

The difference in the loss function for $V$ and $V'$ is

\begin{equation*}
  \begin{aligned}
    \sum_i^k & \left| V' \cdot (X_i,1) - y_i \right|_p -  \sum_i^k \left| V \cdot (X_i,1) - y_i \right|_p \\
    & = \sum_i^k \left| V \cdot (X_i,1) - y_i - \epsilon_i \right|_p -  \sum_i^k \left| V \cdot (X_i,1) - y_i \right|_p  &\quad\text{(from Equation (\ref{epsilon-definition}))} \\
    & = \sum_i^k \max(\left| V \cdot (X_i,1) - y_i \right|_p, \left| \epsilon_i \right|_p) -  \sum_i^k \left| V \cdot (X_i,1) - y_i \right|_p &\quad\text{(ultrametricity)} \\
    & =  \sum_i^k \left| V \cdot (X_i,1) - y_i \right|_p -  \sum_i^k \left| V \cdot (X_i,1) - y_i \right|_p  &\quad\text{(using Equation (\ref{epsilon-setup}))} \\
    & = 0
  \end{aligned}
\end{equation*}

Thus, if the best model at the moment is $V$, there is no ``small
update'' that could be made in any direction where it would be
possible to see an improvement in the loss function.

Because every $p$-adic ball is a plateau of the loss surface, the
gradient (indeed any signal based on a first-order derivative)
vanishes everywhere. Gradient-based optimisers therefore have nothing
to latch onto: the only way to improve a model is to make
discrete jumps that leave the current ball entirely.

\section{Hyperplane Intersection Theorem}

In this section we will show that an affine
function of $n$ variables that minimises a $p$-adic loss function
will pass through at least $n+1$ points in the dataset (assuming the dataset
has at least $n+1$ points and spans $n$ dimensions). This theorem is the key that allows
us to create an algorithm for solving $p$-adic linear regression problems, and
to reason about such problems.



\label{core-theorem-section}
\begin{thm}
  \label{core-theorem}
  Let $n, k \in \mathbb{Z}^+$ where $k \ge n+1$.
  Let $X_1, X_2, \ldots X_k \in \mathbb{Q}^n$ and
  $y_1, y_2, \ldots y_k \in \mathbb{Q}$, where
$y_i \ne y_j \implies X_i \ne X_j$.
Suppose that the data set $X_1, \ldots, X_k$ is non-degenerate, that is, 
there is no non-zero affine function $\phi : \mathbb{Q}^n \rightarrow \mathbb{Q}$ for which $\phi(X_i) = 0$ for all $i$. Then there is
an affine function $M~:~\mathbb{Q}^n \rightarrow \mathbb{Q}$ which minimises $\sum_{i=1}^{k} \left|M(X_i) - y_i\right|_p$, and
$M(X_i) - y_i = 0$ for at least $n+1$ distinct values of $i$.
Moreover, all such optimal affine functions have the latter property.

\end{thm}

\begin{proof}

We will sometimes use the term {\it residual} of a data point $(X_i,y_i)$ 
with respect to an affine function $F$ to mean the quantity 
$F(X_i) - y_i$.

The main part of the proof is devoted to establishing the following
key assertion:

\begin{quote}
    For every affine function $F~:~\Q^n \rightarrow \Q$, for which
    the number $m$ of values of $i$ for which $F(X_i) - y_i = 0$
    satisfies $0 \le m < n+1$, there exists an affine function
    $G~:~\Q^n \rightarrow \Q$ such that
    $\sum_{i=1}^k \left|G(X_i) - y_i \right|_p < 
     \sum_{i=1}^k \left|F(X_i) - y_i \right|_p$,
    and there are more than $m$ data points $(X_i,y_i)$ whose
    residual with respect to $G$ vanishes.
\end{quote}

  Let $F~:~\Q^n \rightarrow \Q$, with $F(X) = V \cdot (X,1)$,
  be an affine function, and suppose that the number $m$ of values of $i$
  for which $F(X_i) - y_i = 0$ satisfies $0 \le m < n+1$.
  Observe that we have flexibility in choosing the order of the
  elements $X_i$. So without loss of generality, we can assume that
  $F(X_i) - y_i = 0$ when $i \leq m$  and $F(X_i) - y_i \neq 0$
  otherwise. Equivalently $V \cdot (X_i,1) - y_i = 0$ when $i \leq m$ and
  $V \cdot (X_i,1) - y_i \neq 0$ otherwise.

  Furthermore, note that we have not yet specified the order
  of the remaining elements, as we will do so later.


  Let us create a vector $V' \in \mathbb{Q}^{n+1}$
  with the goal of making another solution $(V + V')$ which has a lower loss.
  A good place to start would be to make sure
  we don't affect the value of the function at the $m$ points that are already correct.

  Formalising that idea, we would want
  $(V + V') \cdot (X_i, 1) - y_i = 0$ when $i \leq m$.
  Since $V \cdot (X_i,1) - y_i = 0$ when $i \leq m$, this is looking for
  a $V' \neq 0$ satisfying $V' \cdot (X_i, 1) = 0$.


  This is indeed possible.
  We would be solving the simultaneous equations defined by this matrix calculation, looking for a $V' \neq 0$.
  \begin{equation}
    (V'_1, V'_2, \ldots, V'_{n+1})
    \left(\begin{array}{ccc}
            X_{1,1} & \dots & X_{m,1} \\
            \vdots &       &        \vdots \\
            X_{1,n} & \dots & X_{m,n}  \\
            1 & \dots &  1 \\
          \end{array} \right) = (0,0, \ldots, 0)
        \label{homogenous-equations}
      \end{equation}

      There are $n+1$ unknowns $V'_1 \ldots V'_{n+1}$ and $m$ homogeneous equations, meaning that not only
      is there a guarantee of a non-zero solution, there
      are going to be at least $n+1-m$ components of $V'$ that can be chosen freely.

      Choose an arbitrary non-zero $V'$ satisfying Equation (\ref{homogenous-equations}). Since we
      have assumed that the $X$-data-set is non-degenerate, there
      exists $i$, so that $m < i \leq n+1$ where $V' \cdot (X_i, 1) \neq 0$.

      Observe that if $\alpha \in \mathbb{Q}$ then $\alpha V'$ is also a solution; that is we have:

      \begin{equation}
        (V + \alpha V') \cdot (X_i, 1) - y_i = 0 \text{ when } i \leq m.
        \label{alpha-solution}
      \end{equation}

      In order to construct the desired $G$, 
      and hence to prove our key assertion, 
      we would like to select one more $(X_i, y_i)$
      pair and make it have a residual zero with respect to some function that also keeps the residual zero for the first $m$ points.

      For each $i$ in the range $m < i \leq n + 1 $ where $V' \cdot
      (X_i,1) \ne 0$, let us define $\alpha_i$ as follows

      \begin{equation}
        \alpha_i = \frac{V \cdot (X_i, 1) - y_i}{- V' \cdot (X_i, 1)}
        \label{alpha-definition}
      \end{equation}

      Observe from Equations (\ref{alpha-solution}) and 
      (\ref{alpha-definition}) that when $\alpha_i$ is defined:

      \begin{equation}
        (V + \alpha_i V') \cdot (X_j, 1) - y_j = 0 \text{ when } j \leq m \text{ or when } j = i.
      \end{equation}

      We can now decide on an ordering for the data elements $(X_i,y_i)$ from $m+1 \le i \le k$
      which we had previously left unspecified.

      Select the $\alpha_i$ with the smallest $p$-adic absolute
      value. If multiple candidates share this minimal value, break the
      tie at random --- it makes no difference to the remainder of the proof.

      Let the corresponding data element $(X_i, y_i)$ be element $m+1$.

      We observe that the remaining data elements don't need any particular
      ordering, but for convenience in the proof we will sort them a little
      further.  Let us put all the data elements for which $\alpha_i$ is 
      defined next,
      and the data elements for which $\alpha_i$ is not defined last. Let $s$ be the last data element for which $\alpha_i$ is defined.
      The following will be true:

      \begin{equation}
        m + 1 < i \le s \implies \left| \alpha_{m+1} \right|_p \leq \left| \alpha_i \right|_p
        \label{minimum-alpha}
      \end{equation}

      We can now calculate the loss for the solution $V + \alpha_{m+1} V'$.

      First,
      let us break it up into the ranges: the already-zero-residual points, $1 \le i \le m$,
      the index of the newly-chosen element $i = m+1$,
      the range $m+2 \le i \le s$ and the range $s < i \le k$:

  \begin{equation}
    \begin{split}
      \sum_{i=1}^k  \left| (V + \alpha_{m+1} V') \cdot (X_i,1) - y_i \right|_p & =  \sum_{i=1}^m \left| (V + \alpha_{m+1} V') \cdot (X_i,1) - y_i \right|_p \\
      & + \left| (V + \alpha_{m+1} V') \cdot (X_{m+1},1) - y_{m+1} \right|_p \\
      & + \sum_{i=m+2}^s \left| (V + \alpha_{m+1} V') \cdot (X_i,1) - y_i \right|_p \\
      & + \sum_{i=s+1}^k \left| (V + \alpha_{m+1} V') \cdot (X_i,1) - y_i \right|_p
    \end{split}
    \label{quadpartite}
  \end{equation}

  By construction, the first two ranges sum to zero. In particular, note that for the second ``range'' there is a strict inequality:

  \begin{equation}
     \left| (V + \alpha_{m+1} V') \cdot (X_{m+1},1) - y_{m+1} \right|_p = 0 <  \left| V \cdot (X_{m+1},1) - y_{m+1} \right|_p
    \label{strictless}
  \end{equation}

  For the third range of Equation (\ref{quadpartite}) we can use the
  strong triangle inequality to break it apart.
  
  \begin{equation}
    \begin{split}
      \sum_{i=m+2}^s & \left| (V + \alpha_{m+1} V') \cdot (X_i,1) - y_i \right|_p \\
      &  =       \sum_{i=m+2}^s \left| V \cdot (X_i,1) + \alpha_{m+1} V' \cdot (X_i,1) - y_i \right|_p \\
      &  \le \sum_{i=m+2}^s \max(\left| \alpha_{m+1} V' \cdot (X_i,1) \right|_p,  \left| V \cdot (X_i,1) - y_i \right|_p) \\
      \end{split}
      \label{sum-of-the-non-zero-tail}
    \end{equation}

    Focussing on the first term of the $\max$ for an arbitrary $i$, we can use Equation (\ref{minimum-alpha}) to put a bound on its size, and Equation (\ref{alpha-definition}) to expand and then simplify:

    \begin{equation}
      \begin{split}
        \left| \alpha_{m+1} V' \cdot (X_i,1) \right|_p & \le \left| \alpha_{i} V' \cdot (X_i,1) \right|_p \\
        & = \left| \frac{ V \cdot (X_i,1) - y_i}{- V' \cdot (X_i,1)} (V' \cdot (X_i, 1)) \right|_p \\
        & = \left| V \cdot (X_i, 1) - y_i \right|_p
      \end{split}
      \label{bounding}
    \end{equation}

    Notice that the expression on the last line of Equation (\ref{bounding}) is exactly the same as the second term of the $\max$ in 
    Equation (\ref{sum-of-the-non-zero-tail}). This lets us simplify the $\max$ considerably.

    \begin{equation}
      \begin{split}
        \sum_{i=1}^s  \left| (V + \alpha_{m+1} V') \cdot (X_i,1) - y_i \right|_p & \le \sum_{i=m+2}^k \max(\left| V \cdot (X_i, 1) - y_i \right|_p ,  \left| V \cdot (X_i,1) - y_i \right|_p) \\
        & = \sum_{i=m+2}^s \left| V \cdot (X_i, 1) - y_i \right|_p
      \end{split}
      \label{non-zero-tail-lesser}
    \end{equation}

    Finally, consider the fourth range of Equation
    (\ref{quadpartite}), where $\alpha_i$ could not be defined because $V' \cdot (X_i,1) = 0$. This equality holds:

    \begin{equation}
      \begin{split}
        \sum_{i=s+1}^k \left| (V + \alpha_{m+1} V') \cdot (X_i,1) - y_i \right|_p & =
        \sum_{i=s+1}^k \left| V \cdot (X_i,1) - y_i  + \alpha_{m+1} V' \cdot (X_i,1) \right|_p  \\
              &=  \sum_{i=s+1}^k \left| V \cdot (X_i,1) - y_i \right|_p \\
      \end{split}
      \label{sum-of-the-zero-tail}
    \end{equation}

    We can now compare the loss of the function $F$ 
    in the key assertion with the loss
    of the function specified by $(V + \alpha_{m+1} V')$. 
    We can substitute in the inequalities
    from Equations (\ref{strictless}), (\ref{non-zero-tail-lesser}) and
    (\ref{sum-of-the-zero-tail}) into Equation (\ref{quadpartite}), to obtain 
    the following inequality:

    \begin{equation}
      \begin{aligned}
        \sum_{i=1}^k & \left| (V + \alpha_{m+1} V') \cdot (X_i,1) - y_i \right|_p  \\
        & < \left| V \cdot (X_{m+1},1) - y_{m+1} \right|_p + \sum_{i=m+2}^s \left| V \cdot (X_i, 1) - y_i \right|_p + \sum_{i=s+1}^k \left| V \cdot (X_i, 1) - y_i \right|_p \\
        & =  \sum_{i=1}^k \left| V \cdot (X_i,1) - y_i \right|_p \\
      \end{aligned}
      \label{contradiction-equation}
    \end{equation}

    We put $G(X) = (V + \alpha_{m+1} V') \cdot (X,1)$.
    We have demonstrated above that $G$ has lower loss than $F$ has,
    and $G$ passes through more than $m$ points of the dataset.
    Our key assertion is proved.

    Now we can define and verify our optimal affine function $M$.
    Observe that, by the non-degeneracy assumption,
    there are at least $n+1$ distinct data points $(X_i, y_i)$; and
    the number of non-zero affine functions $H$ which pass through
    at least $n+1$ distinct data points $(X_i, y_i)$ is finite and positive,
    at most $\binom{k}{n+1}$.
    Therefore, there is an affine function $M$ which has least loss
    amongst all such functions $H$. We claim that $M$ is optimal
    amongst {\em all} affine functions.
    This claim is proved as follows.
    Let $F~:~\Q^n \rightarrow \Q$ be an arbitrary affine function,
    and denote by $m$ the number of values of $i$ for which $F(X_i) - y_i = 0$.
    If $m \ge n+1$, then the loss of $M$ is at most that of $F$,
    by definition of $M$. Suppose, on the other hand, that $0 \le m < n+1$.
    Then repeated application of our key assertion to $F$,
    in which we reset $F$ to be the newly constructed $G$ in each
    subsequent step, yields after a finite number of steps
    an affine function, say $H~:~\Q^n \rightarrow \Q$,
    whose loss is lower than that of our original $F$
    and which passes through at least $n+1$ distinct data points $(X_i, y_i)$.
    It follows that the loss of $M$ is less than that of $F$.
    Every optimal affine function must pass through at least $n+1$
    distinct data points, by the key assertion.
  \end{proof}

  Note that we did not make use of any property of the $p$-adic numbers beyond
  satisfying the Strong Triangle Inequality. Thus we observe the following remark:

  \begin{remark}
    The proof of Theorem \ref{core-theorem} generalises directly to any ultrametric field. The calculation of
    $\alpha_i = \frac{V \cdot (X_i, 1) - y_i}{- V' \cdot (X_i, 1)}$ required multiplicative
    inverses. Equation (\ref{sum-of-the-non-zero-tail}) required the Strong Triangle
    Inequality. Everything else was simple algebra over a field.
  \end{remark}

\section{Polynomial corollary}\label{poly-corollary-section}

Of perhaps more interest to number theorists, there is
a simple corollary of Theorem \ref{core-theorem}.

\begin{corollary}
\label{poly-corollary}
Let $k \ge n+1$,
let $x_1, x_2, \ldots, x_k \in \mathbb{Q}$,
with the cardinality of the set $\{x_i~|~1 \le i \le k\} $
at least $n+1$, and let $y_1, y_2, \ldots, y_k \in \mathbb{Q}$.
Suppose $y_i \neq y_j \implies x_i \neq x_j$.
Let $P(x) \in \mathbb{Q}[x]$ be a rational polynomial of degree at most $n$
that minimises
$\sum_{i=1}^k \left|P(x_i) - y_i\right|_p$ amongst all such polynomials.
Then there are at least $n+1$ values of $i$ for which $P(x_i) = y_i$.
\end{corollary}

\begin{proof}
For every polynomial $A(x) \in \mathbb{Q}[x]$ of degree at most $n$,
with $A(x) = a_n x^n + a_{n-1} x^{n-1} + \cdots + a_1 x + a_0$,
there is an associated affine function 
$F_A~:~\mathbb{Q}^n \rightarrow \mathbb{Q}$
defined by 
$F_A(z_1, \ldots, z_n) = a_n z_n + a_{n-1} z_{n-1} + \cdots + a_1 z_1 + a_0$. In fact, the mapping $A \mapsto F_A$ is clearly a one-to-one correspondence between the set of all rational polynomials of degree at most $n$ and the set of all rational valued affine functions with domain $\mathbb{Q}^n$.

For $1 \le i \le k$, put $X_i = (x_i, x_i^2,\ldots, x_i^{n-1}, x_i^n)$.
By assumption, $P(x) \in \mathbb{Q}[x]$ has degree at most $n$ and minimises
$\sum_{i=1}^k \left|P(x_i) - y_i\right|_p$ amongst all such polynomials.
Since for every $A(x) \in \mathbb{Q}[x]$ of degree at most $n$, 
$A(x_i) = F_A(X_i)$ for all $i$, and in light of the one-to-one
correspondence $A \mapsto F_A$ noted above, it follows that $F_P$ minimises
$\sum_{i=1}^k \left|F_P(X_i) - y_i\right|_p$ amongst all affine functions from $\mathbb{Q}^n$ to $\mathbb{Q}$. Moreover, we may observe that $y_i \neq y_j \implies X_i \neq X_j$, by our assumption $y_i \neq y_j \implies x_i \neq x_j$ and the construction of the vectors $X_i$.

Therefore, by Theorem 1, there are at least $n+1$ values of $i$
for which $F_P(X_i) = y_i$, or the $X$-data set is degenerate in the sense of Theorem 1. Suppose for the sake of contradiction that the latter is the case. By our assumption that the cardinality of the set $\{x_i\}$ is at least $n+1$, we may renumber the $x_i$ (and $X_i$) values so that the first $n+1$ of them are pairwise distinct. From the degeneracy assumption, it follows that the (possibly smaller) data set $X_1, X_2, \ldots, X_{n+1}$ is also degenerate. Yet consider the matrix $M$ whose column vectors are $(1, X_i)^T$, for $1 \le i \le n+1$: $M$ is a square Vandermonde matrix, of size $n+1$, in $x_1, x_2, \ldots, x_{n+1}$. By a well known classical theorem, the determinant of $M$ is the product of the differences 
$\prod_{1 \le i < j \le n+1}(x_i  - x_j)$. Hence, by the pairwise distinctness of the first $n+1$ $x_i$ values, this determinant is nonzero. This contradicts the degeneracy of $X_1, X_2, \dots, X_{n+1}$.
The desired conclusion has been proved. $\Box$
  
  

\end{proof}

Similar results can be derived for polynomials of multiple variables.

There is a small extension of Corollary \ref{poly-corollary}.
Suppose we have a polynomial $P(x) \in \mathbb{Q}[x]$ of arbitrary degree,
typically greater than $n$,
which we wish to approximate by a rational polynomial of degree at most
$n$.

Suppose further that a ``good'' polynomial approximation is one that minimises the sum of the $p$-adic differences between the two polynomials at a finite number $k$ of rational points, where $k \ge n+1$.
Let $Q(x) \in \mathbb{Q}[x]$ be an optimal approximation for $P(x)$ in this sense. The following result provides a lower bound on the number of distinct zeros of $P(x) - Q(x)$.

\begin{thm}
\label{poly-approxy}
Let $P(x)$ and $Q(x)$ be rational polynomials,
with $Q(x)$ of degree at most $n$. Suppose that $Q(x)$ is
an optimal $p$-adic approximator
for $P(x)$ at a finite set $S$ of rational points with $|S| \ge n+1$.
Then $P(x) - Q(x)$ has at least $n+1$ distinct zeros in $S$.
\end{thm}

\begin{proof}
Let $k = |S|$, denote the (distinct) elements of $S$ by 
$\{x_1, x_2, \ldots, x_k\}$,
and put $y_i = P(x_i)$, for $1 \le i \le k$.
By assumption, $Q(x)$ is of degree at most $n$ and is an optimal
$p$-adic approximator for $P(x)$ at $S$; that is,
$Q(x)$ minimises $\sum_{i=1}^k |Q(x_i) - y_i|_p$ over all rational
polynomials of degree at most $n$.

By Corollary \ref{poly-corollary}, there are at least $n+1$ points $x_i$
of $S$ such that $Q(x_i) = y_i = P(x_i)$.
The desired conclusion follows immediately. $\Box$


\end{proof}

Let us define a \textit{residual polynomial}\label{polynomial-residual} of $P(x)$ with respect to the prime number $p$, the approximation degree
$n$ and the finite evaluation dataset $S \subset \mathbb{Q}$, with $|S| \ge n+1$, to be a polynomial $P(x) - Q(x)$, where $Q(x)$ is a rational polynomial of degree at most $n$ that minimises the sum of the $p$-adic differences between $P(x)$ and $Q(x)$ at the elements of $S$.

\begin{corollary}
A polynomial $R(x)$ of degree $n+1$
cannot be a residual polynomial of $P(x)$ with respect to
$p$, $n$ and $S$, with $|S| \ge n+1$, if $R(x)$ has a multiple root.
\end{corollary}
\begin{proof}
Let $R(x) \in \mathbb{Q}[x]$ have degree $n+1$.
We prove the contrapositive of the stated claim about $R(x)$.
Suppose that $R(x)$ is a residual polynomial of $P(x)$ with respect to $p$,
$n$ and $S$, with $|S| \ge n+1$.
Then $R(x) = P(x) - Q(x)$, for some $Q(x) \in \mathbb{Q}[x]$ of degree
at most $n$ that minimises the sum of the $p$-adic differences between $P(x)$ and $Q(x)$ at the elements of $S$. By the previous theorem,
$R(x)$ has at least $n+1$ distinct roots. Since the degree of $R(x)$ is $n+1$, this accounts for all the roots, each of which must be simple (i.e. non-multiple), by the fundamental theorem of algebra. Hence $R(x)$ has no multiple root. $\Box$
\end{proof}

\begin{corollary}
Suppose that the degree of $P(x)$ is $n+1$.
Then no residual polynomial of $P(x)$ with respect to $p$, $n$ and $S$, with $|S| \ge n+1$, has an irrational root.   
\end{corollary}
\begin{proof}
We prove a statement which is logically equivalent to the stated claim.
Let $R(x)$ be a residual polynomial of $P(x)$ with respect to $p$, $n$ and $S$, with $|S| \ge n+1$. Then $R(x) = P(x) - Q(x)$, for some
$Q(x) \in \mathbb{Q}[x]$ of degree at most $n$ that minimises the
sum of the $p$-adic differences between $P(x)$ and $Q(x)$ at the
elements of $S$. Since the degree of $P(x)$ is $n+1$ and the degree of $Q(x)$ is at most $n$, $R(x)$ has degree $n+1$. Moreover, by the 
previous theorem, $R(x)$ has at least $n+1$ distinct rational roots.
By the fundamental theorem of algebra, this accounts for all the roots,
none of which is irrational. $\Box$
\end{proof}

\section{Implications of the Hyperplane Intersection Theorem for Machine Learning}

An attribute of $p$-adic metrics 
for hierarchical data is that they naturally respect the discrete, branching 
nature of hierarchical relationships. While Euclidean metrics treat 
space as continuous and uniformly connected, $p$-adic metrics capture 
the ``all-or-nothing'' nature of hierarchical relationships --- either two 
points share a common ancestor at a particular level, or they don't.

This suggests that many machine learning problems involving
hierarchical data - from biological classification to natural language
processing to organizational hierarchies --- might be better
approached using $p$-adic metrics rather than traditional Euclidean
approaches. Our applications to linguistic analysis in Section
\ref{applications} demonstrate this advantage empirically, achieving
better results than Euclidean methods do.

Our proof that optimal $p$-adic regression planes must pass 
through data points reflects a deeper truth: in hierarchical data, 
interpolation between points often makes less sense than selecting 
actual observed points as representatives.
\newcommand{\polyaside}[0]{(or low degree polynomials)}
\subsection{Algorithm}

For low dimensional hyperplanes \polyaside{} and small
datasets a brute force algorithm for multivariate $p$-adic linear regression may be practical, in light of Theorem \ref{core-theorem}: try every relevant--sized subset of observed points and use them as representatives.

For example, consider the case $n = 1$. By Theorem \ref{core-theorem}, finding the line that minimises the $p$-adic residual sum
can be done using $O(r^3)$ operations (where $r$ is the number of elements in the
dataset): for every pair of points in the dataset, of which we may form $O(r^2)$ such pairs,
calculate the line between them, and then for every point calculate the residual. Thus, we may obtain the $p$-adic residual sums for the $O(r^2)$ candidate lines using $O(r^3)$ operations in total. The desired minimizing line is then found by a straightforward pass through the candidate lines.

The brute-force algorithm sketched above for the case $n = 1$ rapidly gets impractical in higher dimensions. An $(n+1)$-dimensional dataset of $r$ elements would need
$O(r^{n+2})$ operations --- and the operations themselves involve
finding divisors and remainders of potentially large numbers.

\subsection{Large primes}

There is an optimisation that can be made when $p$ is large,
which relies on Theorem \ref{big-primes}.

\begin{thm}
  \label{big-primes}
  For any finite dataset $D$ with elements in $\mathbb{Q}^n$, there exists a prime $q$ such that for all primes $p \ge q$,
  the $p$-adic residual for a point of an optimal $p$-adic linear
  regression line (or hyperplane) is either 0 or 1.
\end{thm}
\begin{proof}
  In the degenerate case where all the points in $D$ have one
  coordinate set to the same value (for example, finding the line of best fit
  when all points have the same $x$ value),
  the optimal line or hyperplane will pass through all points and their residuals
  will be 0.

  In the non-degenerate case, a line or hyperplane  will have a finite
  gradient in each coordinate. These gradients will be finite and rational, and therefore
  the residuals will be rational.  There are a finite number of points in
  the dataset, meaning that the residuals form a finite set of finite
  rational numbers.

  Residuals that are zero have a $p$-adic distance of zero.

  Considering the residuals that are non-zero, they define a finite
  set of (integer) numerators and (integer and non-zero) denominators.
  The prime factors of these numerators and their corresponding
  denominators form a finite set, which means that there is a largest
  factor that appears in the set.

  Let the next largest prime be $q$.  Any prime larger greater than or
  equal to $q$ divides no numerator or denominator in the set of
  non-zero residuals. By definition, the $p$-adic distance to any of
  these non-zero residuals is 1.
\end{proof}



For these ``large'' primes (primes $p$ greater than the largest factor in any residual
of the dataset),
the optimal
$p$-adic line or hyperplane will be the one or ones that
pass through the most points.

Point--hyperplane intersection can be calculated
in $O(r^{n+1})$ time by using the
equation of the hyperplane through $n+1$ points as the key
into a hash table. Incidentally, one such calculation is sufficient for all
of these ``large'' primes.



\subsection{Optimisations for polynomial approximation}

We consider a slight variation of 
the polynomial approximation task as per Section
\ref{poly-corollary-section}.
Let $S = \{x_i~|~ 1 \le i \le n+1\}$ be a set of pairwise distinct
rational numbers, and $T = \{y_i~|~1 \le i \le n+1\}$ a corresponding
set of rationals. By polynomial interpolation,
there is a unique rational polynomial $P(x)$ of degree at most $n$
such that $P(x_i) = y_i$, for all $i$.
We call $P(x)$ the associated polynomial {\em interpolant} (for $S, T$).
In one special but not particularly rare case, no search for an
optimal approximation polynomial for $P(x)$ of degree at most $n-1$ needs to be done, since all solutions are equivalent.

\begin{thm}
   Suppose that, for all indices $i, j$, with $i < j$, we have
   $\left| x_i - x_j \right|_p = 1$ (or any other constant),
   and that the associated polynomial interpolant $P(x)$ has degree exactly $n$. 
   Then there are $n+1$ equivalent
    polynomials of degree at most $n-1$ optimally approximating the
   dataset $S, T$ (hence $P(x)$), all of
   which have the same $p$-adic sum of residuals.
\end{thm}

\begin{proof}
    By Theorem \ref{poly-approxy}, a residual 
    polynomial formed from $P(x) - Q(x)$, where $Q(x)$ is an optimal approximation
    polynomial of $P(x)$ of degree at most $n-1$,  has at least $n$
    distinct zeros in $S$. Since this residual has degree exactly $n$,
    by our assumption on the degree of $Q(x)$,
    this residual has exactly $n$ distinct roots in $S$, by the fundamental theorem of algebra. Thus, since $|S| = n+1$, we may associate with $Q(x)$ the unique
    element, say $x_j$, of $S$ for which the residual polynomial does
    not vanish. As a partial converse, 
    if we are given a subset of $S$ of cardinality $n$,
    then there is a polynomial $R(x)$ of degree $n$
    having the elements of this subset as its roots
    and having the same leading coefficient as $P(x)$. We may therefore put
    $Q(x) = P(x) - R(x)$, and observe that $Q(x)$ has degree at most $n-1$.
    In other words, since $S$ consists of $n+1$ points, 
    we can index each potential optimal residual polynomial
    by the point of $S$ at which it is non-zero, and use that
    to index the associated potential optimal approximating polynomial.

    More explicitly, define $R_j(x)$ to be the potential
    optimal residual polynomial 
    which is non-zero at $x_j$ and has the same leading coefficient,
    $a$ say, as $P(x)$:
\[
    R_j(x) = a\prod_{i=1,i \neq j}^{n+1} (x - x_i).
\]

    We then define $Q_j(x) = P(x) - R_j(x)$ as the 
    potential optimal approximating polynomial of degree at most $n-1$ that yields $R_j(x)$ as its residual.
    In summary, we have shown that there are {\em at most} $n+1$ optimal approximating polynomials for $P(x)$ at $S$,
    each of which corresponds to one of the potential
    optimal residual polynomials defined above. 
    We shall show that there are, in fact, {\em exactly} $n+1$ optimal
    approximating polynomials, all of which have the same $p$-adic sum of residuals.

    Observe that $\sum_{c=1}^{n+1} \left| R_j(x_c)\right|_p$ is the $p$-adic sum of residuals for the potential optimal approximating polynomial $Q_j(x)$ for $P(x)$ at $S$.

    Let us consider the difference between the $p$-adic sum of residuals
    for any two such polynomials at $S$.

    Take two indexes $j, k$ and observe that
\[    
 \sum_{c=1}^{n+1} (\left|
    R_j(x_c) \right|_p - \left| R_k(x_c)
    \right|_p)
    = |a|_p \sum_{c=1}^{n+1} \left( \left| \prod_{i=1,i \neq j}^{n+1}  
      (x_c - x_i) \right|_p -  \left| \prod_{i=1,i \neq k}^{n+1} (x_c - x_i) \right|_p
    \right)
\]

 When $c \neq j$ and $c \neq k$, there will be a zero term in one of entries in each product, making it zero. So
  the sum reduces to just the $c = j,k$ terms:

  \begin{equation*}
    \begin{split}
 \sum_{c=1}^{n+1} (\left|
    R_j(x_c) \right|_p - \left| R_k(x_c)
    \right|_p)
    = &
    |a|_p \left| \prod_{i=1,i \ne j}^{n+1} (x_i - x_j) \right|_p +
    |a|_p \left| \prod_{i=1,i \ne j}^{n+1}  (x_i - x_k) \right|_p \\
    & - |a|_p \left| \prod_{i=1,i \ne k}^{n+1} (x_i - x_j) \right|_p -
    |a|_p \left| \prod_{i=1,i \ne k}^{n+1} (x_i - x_k) \right|_p
    \end{split}
\end{equation*}

In the second term, when $i=k$, the
product is zero. Likewise, in the third term when $i=j$. So that reduces to:

\[
 \sum_{c=1}^{n+1} (\left|
    R_j(x_c) \right|_p - \left| R_k(x_c)
    \right|_p) =
    |a|_p \left| \prod_{i=1,i \ne j}^{n+1} (x_i - x_j) \right|_p  -
  |a|_p \left| \prod_{i=1,i \ne k}^{n+1} (x_i - x_k) \right|_p
\]

   Using the widespreadness property ( $\forall i,j \left| x_i - x_j \right|_p = 1 $ ), this becomes:

\[   
 \sum_{c=1}^{n+1} (\left|
    R_j(x_c) \right|_p - \left| R_k(x_c)
    \right|_p)     
    =
   (|a|_p \prod_{i=i,i \ne j}^{n+1} 1) - (|a|_p \prod_{i=i,i \ne k}^{n+1} 1) \\
    = |a|_p  - |a|_p 
    = 0 
\]

This demonstrates the remaining parts of the theorem's statement, namely, that the potential optimal approximating polynomials of degree at most
$n-1$ for $P(x)$ at $S$
all have the same $p$-adic sum of residuals, and hence that there are
{\em exactly} $n+1$ optimal approximating polynomials of degree
at most $n-1$ for $P(x)$ at $S$.
\end{proof}

\section{Applications}\label{applications}

To the best of the authors' knowledge, no applications for
$p$-adic linear regression have been found other than
the ones in this section. 

We expect that non-linear machine learning techniques
will enable many more applications beyond the two
outlined here.

\subsection{A slightly-contrived multivariate example}

The first application makes use of the hierarchial structure of the
WordNet \mycite{wordnet} ontology. We can use this to give unique
$p$-adic values to word senses. This lets us find correlations between
collections of objects even in the presence of some randomness by
creating a multi-variate $p$-adic linear regression problem, solving
it and using the coefficients of the linear model to gain insight into
the relations of the objects.

We can express this in the following problem statement.

Zorgette the alien has come to Earth, and instructed her
robots to collect three examples of different kinds of trees
on a sequence of missions.

Unfortunately, one of her three
robots is faulty --- she does not know which one --- and it collects random objects.

The two robots which are working should be highly correlated in what they collect
on each mission, and the third (the faulty one) highly uncorrelated.

\subsubsection{Turning a Zorgette problem into a linear regression problem}\label{zorgette-to-linear}

Zorgette's problem involves trees --- both mathematically and physically.
WordNet is a large lexical database of English that organises words into sets of synonyms called synsets and encodes various semantic relations between them in the form of a directed graph. A very small amount of edge pruning turns it into
a tree.
A portion of the WordNet 3.1
hierarchy is shown in Figure \ref{wordnetpicture}.

The path to the noun {\sf mammoth.n.01} is {\sf
1.2.3.37.5.4.4.5.3.8.4.17.1.4
}, which can be encoded as 

$1 + 2p + 3p^2 + 37p^3 + 5p^4 + 4p^5 + 4p^6 + 5p^7 + 3p^8 + 4p^9 + 17p^{10} + p^{11} + 4p^{12}$

This encoding has the neat property that the similarity of two nodes (how deep their
closest common ancestor is) can be calculated using their $p$-adic distance. Two nodes are similar
if they are $p$-adically close.

Thus Zorgette wants to set up this $p$-adic linear equation:

\begin{equation*}
a X = b Y + c Z + d    
\end{equation*}

Where $X$, $Y$ and $Z$ are column vectors, with a row for each mission. Each
element is the $p$-adic
WordNet number for the object that the
robot returned
on a given mission. Robot 1's objects are encoded in the $X$ vector,
robot 2's objects as $Y$ and robot 3's objects as $Z$.

Zorgette wants to learn the optimal values of $a$, $b$, $c$ and $d$,
that would minimise the $p$-adic error of that equation on her data
set. The $p$-adic error
corresponds to the semantic similarity of the object that Zorgette's
linear regression predicts versus the actual object, i.e. her
linear regression model should try to predict an object which is
as similar as possible to what robot 1 returned with, based on what
robot 2 or robot 3 brought back.

If robot 2 is faulty,
the objects it will have collected will be
random noise that aren't related to robot 1's or robot
3's souvenirs, so $b$ will be 0. Conversely, if robot 3 is faulty, then $c$ will be
0. If robot 1 is faulty, then both $b$ and $c$ will be 0.

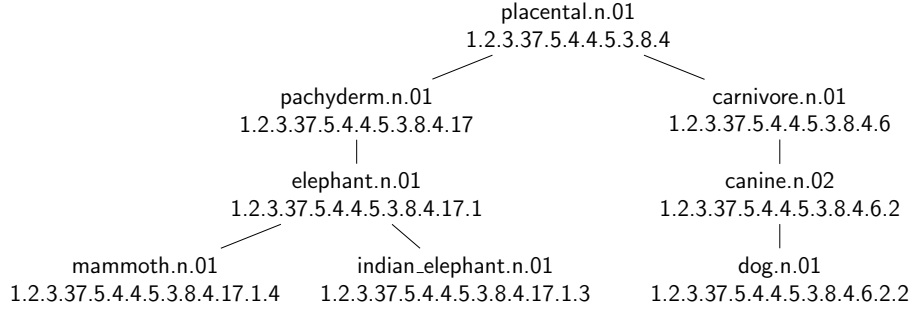
\begin{figure}
  {\sf
    \scalebox{0.75}{
\begin{tikzpicture}
 \node[align=center,xshift=-20mm] {placental.n.01\\1.2.3.37.5.4.4.5.3.8.4}
        child {
            node[xshift=-30mm,align=center] {pachyderm.n.01\\
            1.2.3.37.5.4.4.5.3.8.4.17}
            child {
             node[align=center]{elephant.n.01\\1.2.3.37.5.4.4.5.3.8.4.17.1}
             child {
               node[xshift=-30mm,align=center] {mammoth.n.01\\1.2.3.37.5.4.4.5.3.8.4.17.1.4}
             }
             child {
               node[align=center,xshift=10mm] {indian\_elephant.n.01\\1.2.3.37.5.4.4.5.3.8.4.17.1.3}
             }
             }
        }
        child {
            node[xshift=30mm,align=center] {carnivore.n.01\\1.2.3.37.5.4.4.5.3.8.4.6}
            child {
                node[align=center] {canine.n.02\\1.2.3.37.5.4.4.5.3.8.4.6.2}
                child {
                   node[align=center] {dog.n.01\\1.2.3.37.5.4.4.5.3.8.4.6.2.2}
                }
            }
        };
      \end{tikzpicture}
      }
}
\caption{A portion of the WordNet hierarchy, with a sample encoding for $p>402$; $p$ must exceed the largest child index in the pruned tree, so we take $p=409$}
\label{wordnetpicture}
\end{figure}


\subsubsection{Zorgette's results}

Code for the Zorgette scenario is in \url{github.com/solresol/padicwordnet}.
It includes a randomly-generated set of missions.
The results are in Table \ref{zorgette-catalog}.
The objects are categorized using WordNet 3.1's taxonomy, and encoded
using the smallest prime that can be safely used (409) without causing
clashes.

\begin{table}[t]
  \centering
  \tiny
    \begin{tabular}{p{18mm}p{43mm}p{43mm}p{43mm}}
\toprule
Zorgette's request & Robot 1's loot & Robot 2's loot & Robot 3's loot \\
\midrule
chestnut.n.02 & {\bf japanese\allowbreak\_chestnut.n.01} \newline 273116748704467022682724613459 & {\bf ozark\allowbreak\_chinkapin.n.01} \newline 326691034247600388922020237468 & {\bf strickle.n.02} \newline 45991216075942090948 \\
hornbeam.n.01 & {\bf european\allowbreak\_hornbeam.n.01} \newline 117240465583858939981595536269 & {\bf american\allowbreak\_hornbeam.n.01} \newline 63666180040725573742299912260 & {\bf cleric.n.01} \newline 655934845482986543017862842 \\
hop\allowbreak\_hornbeam.n.01 & {\bf old\allowbreak\_world\allowbreak\_hop\allowbreak\_hornbeam.n.01} \newline 117109477110648344954115595868 & {\bf eastern\allowbreak\_hop\allowbreak\_hornbeam.n.01} \newline 63535191567514978714819971859 & {\bf switchboard.n.01} \newline 1573780139196323304716 \\
beech.n.01 & {\bf american\allowbreak\_beech.n.01} \newline 55675883174879277066023547799 & {\bf copper\allowbreak\_beech.n.01} \newline 162824454261146009544614795817 & {\bf nun's\allowbreak\_habit.n.01} \newline 396171205890659683677595416 \\
necklace\allowbreak\_tree.n.01 & {\bf bead\allowbreak\_tree.n.01} \newline 68643742022728184786537647498 & {\bf jumby\allowbreak\_bead.n.01} \newline 122218027565861551025833271507 & {\bf white\allowbreak\_slave.n.01} \newline 800684989475070496403917474 \\
hackberry.n.01 & {\bf european\allowbreak\_hackberry.n.01} \newline 116847500164227154899155715066 & {\bf american\allowbreak\_hackberry.n.01} \newline 63273214621093788659860091057 & {\bf venetian\allowbreak\_glass.n.01} \newline 1285764896971742062431186 \\
locust\allowbreak\_tree.n.01 & {\bf clammy\allowbreak\_locust.n.01} \newline 120646165887334410696073986695 & {\bf honey\allowbreak\_locust.n.01} \newline 227794736973601143174665234713 & {\bf range.n.02} \newline 5762476220082796694 \\
angiospermous\allowbreak\_tree.n.01 & {\bf bush\allowbreak\_willow.n.02} \newline 375942700174784119254477828244 & {\bf terebinth.n.01} \newline 2840359835158918966262076532658 & {\bf standard\allowbreak\_cell.n.01} \newline 394573092415095127486114211 \\
bonsai.n.01 & {\bf ming\allowbreak\_tree.n.02} \newline 110167088030486808497678754615 & {\bf ming\allowbreak\_tree.n.01} \newline 56592802487353442258383130606 & {\bf vegetable.n.02} \newline 106017242436927074913158021 \\
incense\allowbreak\_tree.n.01 & {\bf gumbo-limbo.n.01} \newline 224912990562968052570106545891 & {\bf elephant\allowbreak\_tree.n.01} \newline 171338705019834686330810921882 & {\bf fumigator.n.02} \newline 99579452998956312316 \\
\bottomrule
\end{tabular}

    \caption{What Zorgette's Robots Fetched, WordNet 3.1, $p=409$}
    \label{zorgette-catalog}
\end{table}

\newcommand{\zorgetteints}[0]{x = y + 53574285543133366239295624009}

Running the regression produces:

\[
\zorgetteints{}
\]

Note that $53574285543133366239295624009 = 409^{11}$, which is a very
small number 409-adically, since it is so highly divisible by 409.
The
variable $x$ (what robot 1 collected) is clearly closely related to
variable $y$ (what robot 2 collected), and completely unrelated to
the variable $z$ (what robot 3 collected). From this Zorgette can
(correctly) observe that robot 3 is faulty.

If Zorgette had taken the integers from Table \ref{zorgette-catalog}
and tried to use ordinary least squares to predict
the optimal coefficients, she would have found:

\newcommand{\zorgetteols}[0]{x = 0.0998903983521872 y - 112.482267940678 z + 1.43578101728206 \cdot 10^{29}}

\[
  \zorgetteols
\]

She would then (incorrectly) assume that robot 2 was faulty.


\subsection{Indo-European Grammar as a Linear Regression Problem}

This subsection is a review of \mycite{aaclpadiclinear}, which is the
only application of $p$-adic linear regression we were able to
find in our literature review. They observe that it is possible to
model the pluralisation of nouns as a machine learning problem: given
a corpus of singular forms and plural forms, the task is to find a linear function
that can form a plural from a previously-unknown singular.

They found that when samples of nouns that are $2$-adically close are
used to train a regressor that tries to predict pluralisation, the
regression often matches the grammar rules for that language. They
reported a Bonferroni-adjusted probability of $3.13 \times 10^{-160}$
in their experiment comparing $p$-adic linear regression with Euclidean
methods across 1500 different human languages.

The ability to analyse grammar rules at scale like this also
turned up the previously overlooked strange pluralisation rules of the
Dobu language --- an Austronesian language (Oceanic, Papuan Tip subgroup) that is known to have been isolated
from Indo-European influence for thousands of years. The strangeness is that despite that isolation, Dobu speakers
pluralise by suffixing in ways that look Indo-European. No explanation for this phenomenon has yet been identified.

\section{Open problems}

Given a small value of $p$, is there any faster algorithm than brute-force searching
through all possible hyperplanes?

Quantum algorithms for finding a minimum in a general dataset (whether
 computed on-the-fly or dynamically) are known \mycite{baritompa}.  That
 algorithm cannot quite achieve a $N / \log N$ speed improvement for
finding the minimum  (where $N$ is the number of possible
values to search through --- $N = r^n$ in this case) because as there are
 fewer and fewer values below the threshold level at each iteration,
 and Grover's algorithm \mycite{grover1996} needs to do more work at
 each level.
 Can the distribution of
 $p$-adic residuals (which has regular periodic local minima) be  
exploited to give better speed improvements still?


It is common in machine learning problems to add
regularising terms to the loss function.  What are
the appropriate regularisation terms to use? When is regularisation
helpful?  How can we solve a regularised $p$-adic linear regression
problem?

Theorem \ref{core-theorem} on Page \pageref{core-theorem} puts an upper
bound on the number of equally good lines of best fit. If
$D = \{(X,y) | X \in \mathbb{R}^n, y \in \mathbb{R} \}$,
then the maximum number of lines of best fit is less than
or equal to
\[
\binom{d}{n+1} = \frac{d!}{(n+1)!(d-n-1)!}
\]
 where $d$ is the cardinality of $D$.
Is this the tightest
upper bound
possible?

Is there any upper bound on the number
of lines of best fit for a given value of $p$?

Is Theorem \ref{poly-approxy} also
true if the approximation is measured at an infinite
number of points?

Is it possible for a polynomial $P(x)$ to have multiple residual
polynomials (as defined in Section \ref{polynomial-residual} on page \pageref{polynomial-residual} with
respect to the same prime and dataset? It seems likely, given that in
simple $p$-adic linear regression multiple equally-good lines of best
fit are possible.

Is it possible for one polynomial to be the residual polynomial for
multiple higher degree polynomials?  This also seems likely. What is
the maximum number of distinct polynomials one polynomial can be a
residual for?

Having rational roots with no duplication is a necessary condition
to be a polynomial residual. Is it a sufficient condition?

\section{Conclusion}

While $p$-adic metrics have been largely overlooked in machine learning, 
our results suggest they may provide valuable insights about properly handling 
hierarchical data. The success of $p$-adic regression in linguistic 
analysis, combined with our theoretical understanding of why it works, 
points to a broader principle: the metric space we choose should match 
the inherent structure of our data.

This opens up new research directions for machine learning on hierarchical 
data structures, from improved algorithms for taxonomic classification 
to better methods for analyzing organizational hierarchies. Future work 
might explore how other machine learning techniques could be reformulated 
in $p$-adic space to better handle hierarchical data.

\section{Acknowledgements}

The authors would especially like to thank Igor Shparlinski
without whom this paper would never have been written,
Micka\"el Montessinos for his corrections and generalisations
and the very insightful comments of the anonymous reviewer.

\printbibliography

\end{document}